\documentclass[conference]{IEEEtran}
\IEEEoverridecommandlockouts

\usepackage{cite}
\usepackage{amsmath,amssymb,amsfonts}
\usepackage{algorithmic}
\usepackage{graphicx}
\usepackage{textcomp}
\usepackage{xcolor}
\usepackage{tikz}
\usepackage{float}
\usepackage{subcaption}

\usepackage{diagbox}
\usepackage{array} 
\usepackage{multirow}
\usepackage{booktabs}

\usepackage{acronym}
\newacro{CSP}{Constraint satisfaction problems}
\newacro{SAT}{Boolean satisfiability}
\newacro{SA}{simulated annealing}
\newacro{SB}{simulated bifurcation}
\newacro{FE}{Fourier expansion}
\newacro{PLE}{parity learning with error}
\newacro{IQR}{interquartile range}
\newacro{DOPO}{degenerate optical parametric oscillator}
\newacro{OIM}{oscillator-based Ising machine}

\usepackage{amsthm}
\newtheorem{lemma}{Lemma}
\newtheorem{theorem}{Theorem}
\newtheorem{corollary}{Corollary}
\newtheorem{proposition}{Proposition}
\newtheorem{remark}{Remark}
\newtheorem{example}{Example}
\newtheorem{definition}{Definition}

\def\BibTeX{{\rm B\kern-.05em{\sc i\kern-.025em b}\kern-.08em
    T\kern-.1667em\lower.7ex\hbox{E}\kern-.125emX}}
\begin{document}

\title{
    Analysis of Higher-Order Ising Hamiltonians

\thanks{
    We thank Moshe Vardi for the helpful and insightful comments and discussions that contributed to this work. 
    This work was supported by the National Research Foundation, Prime Minister’s Office, Singapore, under its Competitive Research Program (NRF-CRP24-2020-0002 and NRF-CRP24-2020-0003), the Ministry of Education (Singapore) Tier 2 Academic Research Fund (MOE-T2EP50220-0012 and MOE-T2EP50221-0008).
    Zhiwei Zhang is supported in part by NSF grants (IIS-1527668, CCF1704883, IIS1830549), DoD MURI grant (N00014-20-1-2787), Andrew Ladd Graduate Fellowship of Rice Ken Kennedy Institute, and an award from the Maryland Procurement Office.
    Corresponding Author: Xuanyao Fong (email: kelvin.xy.fong@nus.edu.sg).
}

}

\author{\IEEEauthorblockN{}
}

\author{\IEEEauthorblockN{Yunuo Cen\IEEEauthorrefmark{1},
                          Zhiwei Zhang\IEEEauthorrefmark{2},
                          Zixuan Wang\IEEEauthorrefmark{1}, 
                          Yimin Wang\IEEEauthorrefmark{1}, and 
                          Xuanyao Fong\IEEEauthorrefmark{1}}
    \IEEEauthorblockA{\IEEEauthorrefmark{1}Department of Electrical and Computer Engineering, National University of Singapore, Singapore}
    \IEEEauthorblockA{\IEEEauthorrefmark{2}Department of Computer Science, Rice University, Houston, TX, USA}
}
\maketitle

\begin{abstract}
It is challenging to scale Ising machines for industrial-level problems due to algorithm or hardware limitations.
Although higher-order Ising models provide a more compact encoding, they are, however, hard to physically implement.
This work proposes a theoretical framework of a higher-order Ising simulator, \texttt{IsingSim}.
The Ising spins and gradients in \texttt{IsingSim} are decoupled and self-customizable. 
We significantly accelerate the simulation speed via a bidirectional approach for differentiating the hyperedge functions.
Our proof-of-concept implementation verifies the theoretical framework by simulating the Ising spins with exact and approximate gradients. 
Experiment results show that our novel framework can be a useful tool for providing design guidelines for higher-order Ising machines.

\end{abstract}


\section{Introduction}
\ac{CSP} are fundamental in mathematics, physics, and computer science. 
The \ac{SAT} problem is a paradigmatic class of \ac{CSP}, where each variable takes values from the binary set \{\texttt{True, False}\}. 
Solving SAT efficiently is of utmost significance in computer science, both from a theoretical and a practical perspective~\cite{kyrillidis2020fouriersat}. 
Numerous problems in various domains are encoded and tackled by SAT solving, e.g., information theory~\cite{golia2022scalable}, VLSI design~\cite{wang2023fastpass}, and quantum computing~\cite{vardi2023solving}. 

A variety of \ac{SAT} problems can be formulated as Ising models~\cite{lucas2014ising}.
As the classical counterpart of quantum computers, 
Ising machines aim to find the minima of Hamiltonian efficiently~\cite{hauke2020perspectives}.
Most Ising machines can be categorized into two types: discrete or continuous.
A discrete Ising spin can only be spin-up or -down.
Discrete Ising machines are mostly based on simulated annealing, which selects and toggles an spin by heuristics~\cite{borders2019integer,aadit2022massively,si2024energy}.
A continuous Ising spin is represented by a continuous variable, \emph{e.g.}, the phases of oscillators.
Continuous Ising machines are mostly based on simulated bifurcation, which relaxes the Hamiltonian function to be continuous, with the global minima encodes the original solutions~\cite{marandi2014network,wang2021solving,wang2024design}.

Encoding combinatorial problems into Ising models often necessitates the use of auxiliary variables to capture the original problem's logic~\cite{lucas2014ising}.
The encoding often increases the hardness, as it leads to a higher dimensionality.
\begin{example}\label{eg:encode}
    Given a set of variables $x\in\{0,1\}^4$, and auxiliary variables $y\in\{0,1\}^3$. 
    Finding an assignment that satisfies the cardinality $x_1+x_2+x_3+x_4\ge{}2$ can be encoded into solving the minima of a discrete function
    $\min_{x,y} (y_2+y_3+y_4-1)^2+(2y_2+3y_3+4y_4-x_1-x_2-x_3-x_4)^2$.
\end{example}

\begin{figure}
    \centering
    \includegraphics[width=0.9\linewidth]{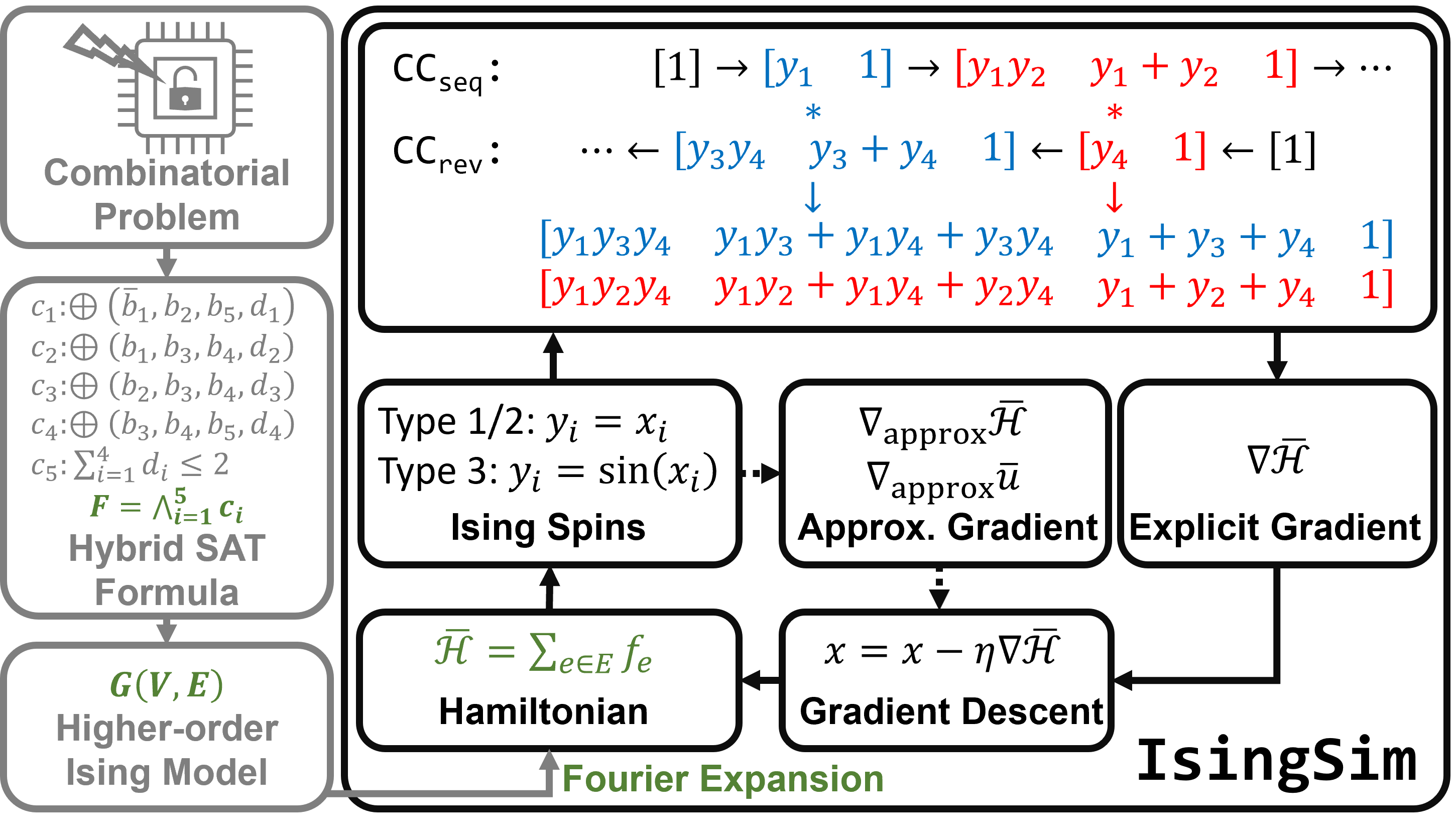}
    \caption{
    \textbf{Left}. Encoding a combinatorial problem, \emph{e.g.}, parity learning with error problem, into a higher-order Ising model.
    \textbf{Right}. The overall flow of \texttt{IsingSim}.
    }
    \label{fig:isingsim}
    \vspace{-0.5cm}
\end{figure}

To alleviate the dimensionality problem, the higher-order Ising model is proposed for compact encoding~\cite{bybee2023efficient, bhattacharya2024unified}. 
Discrete Ising machines mostly toggle a spin when the Hamiltonian function decreases, which works well for conventional Ising models where the edges are only coupling two spins.
For higher-order Ising models, however, toggling a spin might not change the value of the Hamiltonian function.
Consider the cardinality constraint in Example~\ref{eg:encode} and the Ising spins are all spin-down, where toggling any spin does not satisfy the constraint.
On the contrary, continuous Ising machines can make progress as long as the continuous function decreases.

Nevertheless, implementing continuous higher-order Ising machines is challenging, this is because
1) every hyperedge represents a many-body interaction of multiple spins, where the mainstream Ising machines lack of efficient implementation.
2) higher-order Ising model has not been well studied. 
Previous higher-order Ising machines focus on conjunctive normal form (CNF) SAT solving~\cite{bybee2023efficient, bhattacharya2024unified} while failing to generalize to other Boolean constraints.
Hence, we propose a higher-order Ising simulator, referred to as \texttt{IsingSim}, as in Fig.~\ref{fig:isingsim}, to study the behavior of higher-order Ising models. 

\noindent\textbf{Contributions}
We establish the theoretical framework for \texttt{IsingSim}.
We propose a bidirectional approach to differentiate the many-body hyperedge functions efficiently, which can be a guideline for physically implementing higher-order Ising machines.
Additionally, we present and compare three types of Ising spins and gradients in the experimental section, offering insights into the local geometry of the higher-order Ising Hamiltonians.
The Ising spins and gradients are decoupled from the framework, makeing them customizable. 
Our experiments on a real-life benchmark indicate that \texttt{IsingSim} can serve as a useful tool for future Ising machine design.

\section{Theoretical Framework}\label{sec:theory}
\subsection{Higher-Order Ising model}
Given an Ising model, every Ising spin represents a binary variable, and every edge maps two Ising spins $\{-1, 1\}^2$ to a binary value $\{-1, 1\}$. 
The Hamiltonian function is the sum of the edge function.
The Ising model natively encodes Max-2-XOR problems, \emph{i.e.}, $-1$ and $1$ denote Boolean \texttt{true} and \texttt{false}.
We define a higher-order Ising model ${G}(V, E)$ to generalize the Ising model for more general combinatorial optimization, \emph{e.g.}, hybrid SAT problem.
\begin{definition}[Higher-Order Ising Model and Hybrid SAT]\label{def:hoim}
    Let $x=(x_1,...,x_n)$ be a sequence of $n$ Boolean variables.
    A hybrid SAT formula is a conjunction of hybrid constraints, \emph{i.e.}, $F=\bigwedge_{c\in{}C}c$.
    Every hyperedge $e\in{}E$ has a Boolean function $f_e$ encoding a hybrid Boolean constraint $c$, and is mapping from $\{-1,1\}^{|e|}$ to $\{-1,1\}$.
    Then the Hamiltonian function is
    \begin{equation}\label{eq:hoim}
        {\mathcal{H}}(x) = \sum_{e\in{}E} w_e{}f_e(\{x_l: \forall{l}\in{}e\}),
    \end{equation}
    where $w_e$ is the weight of the hyperedge $e$ and $w_e>0$.
\end{definition}
\begin{lemma}[Reduction]\label{lmm:reduction}
    The Boolean formula $F=\bigwedge_{c\in{}C}c$ is satisfiable if and only if
    \begin{equation*}
        \min_{a\in\{-1,1\}^n} {\mathcal{H}}(a) = -\sum_{e\in{}E} w_e.
    \end{equation*}
\end{lemma}
\subsection{Walsh-Fourier Expansion} 
\ac{FE} is a multilinear polynomial representation of a Boolean hyperedge function $f_e$, such that the polynomial agrees with the Boolean function on all Boolean assignments~\cite{o2021analysis}. 

\begin{theorem}[Walsh-Fourier Expansion~\cite{o2021analysis}]\label{thm:FE}
    Given a function $f_e:\{\pm{}1\}^{|e|} \to \{-1,1\}$, there is a unique way of expressing $f_e$ as a multilinear polynomial, called the \ac{FE}, with at most $2^{|e|}$ terms in $S$ according to:
    \begin{equation*}
        f_e(x)=\sum_{S\subseteq[n]}\left(
            \hat{f}_e(S)\cdot\prod_{i\in{}S}x_{e_i}
        \right)
    \end{equation*}
    where $\hat{f}_e(S) \in \mathbb{R}^{2^{|e|}}$ is called Walsh-Fourier coefficient, given $S$, and computed as:
    \begin{equation}\label{eq:coefficient}
        \hat{f}_e(S)=\frac{1}{2^{|e|}}\sum_{x\in\{\pm{}1\}^{|e|}}\left(f_e(x)\cdot\prod_{i\in{}S}x_{e_i}\right)
    \end{equation}
\end{theorem}

The following example shows that the cardinality constraint in Example~\ref{eg:encode} can be transformed by \ac{FE}.

\begin{example}\label{eg:expansion}
    Given a cardinality constraint $e:x_1+x_2+x_3+x_4\ge{}2$, its Walsh-Fourier expansion is
    \begin{align*}
        f_e(x) =& -\frac{3}{8}x_1x_2x_3x_4 
        -\frac{1}{8}(x_1x_2x_3+x_1x_2x_4+x_1x_3x_4\\
        & +x_2x_3x_4) + \frac{1}{8}(x_1x_2+x_1x_3+x_1x_4+x_2x_3 \\
        & +x_2x_4+x_3x_4) +\frac{3}{8}(x_1+x_2+x_3+x_4)-\frac{3}{8}.
    \end{align*}
    
\end{example}

\subsection{Ground States of Hamiltonian Functions}
Via \ac{FE}, the Ising Hamiltonian function can be transformed into a multilinear polynomial. 
Different Ising machine hardware relaxes the discrete domain $\{-1,1\}^n$ to the continuous domain using different methods.
The most straightforward method is to relax the domain to a hypercube $[-1,1]^n$.
Lemma~\ref{lmm:reduction} leads to the following theorem.

\begin{theorem}[Type I Relaxation~\cite{kyrillidis2020fouriersat,cen2023massively}]\label{thm:reduction}
    The Boolean formula $F=\bigwedge_{c\in{}C}c$ is satisfiable if and only if
    \begin{equation*}
        \min_{a\in[-1,1]^n} {\mathcal{H}}(a) = -\sum_{e\in{}E} w_e.
    \end{equation*}
\end{theorem}
\begin{proof}
    Since the Laplacian $\Delta{\mathcal{H}}(a)$ is $0$, the optima of ${\mathcal{H}}(a)$ are on the boundary by the maximum principle~\cite{protter2012maximum}.
\end{proof}

The above theorem can be generalized as follows.

\begin{corollary}[Type II Relaxation]\label{cor:type2}
    The Boolean formula $F=\bigwedge_{c\in{}C}c$ is satisfiable if and only if
    \begin{equation}\label{eq:type2}
        \min_{a\in\left[-\sqrt{p},\sqrt{p}\right]^n} \left({\mathcal{H}}(a)+\sum_{i=1}^n\left(a_i^4-2pa_i^2\right)\right) = -\sum_{e\in{}E} w_e - n,
    \end{equation}
    where the second term is called intrinsic locking in \ac{DOPO}, and $p>0$~\cite{marandi2014network}.
\end{corollary}
\begin{proof}
    The minima of $\left(a_i^4-2a_i^2\right)$ is on $\{\pm{}\sqrt{p}\}$.
    A function retains the minima when it is a linear combination of the functions with the same minima.
\end{proof}

\begin{corollary}[Type III Relaxation]\label{cor:type3}
    The Boolean formula $F=\bigwedge_{c\in{}C}c$ is satisfiable if and only if
    \begin{equation}\label{eq:type3}
        \min_{a\in\mathbb{R}^n} \left({\mathcal{H}}(\sin(a))+\sum_{i=1}^n\left(\cos(2a_i)\right)\right) = -\sum_{e\in{}E} w_e - n,
    \end{equation}
    where second term is called injection locking in \ac{OIM}~\cite{wang2021solving}.
\end{corollary}
\begin{proof}
    By Theorem~\ref{thm:reduction}, the minima of $\mathcal{H}(\cos(a))$ occur when $\sin(a)\in\{\pm{}1\}^n$.
    When $\cos(2a)=\{-1\}^n$, $\sin(a)\in\{\pm{}1\}^n$.
    Hence, the minima of the overall function are attained on $\{\left(n+\frac{1}{2}\right)\pi|n\in\mathbb{Z}\}^n$.
\end{proof}

\begin{remark}\label{rmk:optimality}
    Type I relaxation retains the multilinear properties, where the Hamiltonian function can be proven to have no local optima in the interior points~\cite{kyrillidis2020fouriersat}. 
    Type II and III relaxations are physics-inspired.
    They give less consideration to the local geometry of the Hamiltonian function. 
    The presence of inner local optima could slow down the optimization process.
\end{remark}

\subsection{Scalable Evaluation}  
\begin{definition}[Convolution]
    The linear convolution of $g\in\mathbb{R}^n$ and $h\in\mathbb{R}^m$ is a sequence $(g*h)\in\mathbb{R}^{n+m-1}$.
    Each entry in $(g*h)$ is defined as:
    \begin{equation*}
        (g*h)_i=\sum_{j=0}^{i}g_{i-j}\cdot{}h_j.
    \end{equation*}
\end{definition}
From Example~\ref{eg:expansion} we can observe that, the coefficients of the \ac{FE} depend only on the order of the terms.
Hence Theorem~\ref{thm:FE} can be simplified to the following corollary using convolution.
\begin{corollary}[Symmetric]\label{cor:symmetric}
    Given a symmetric Boolean constraint, by leveraging the symmetric property, the \ac{FE} can be reduced to $|e|+1$ terms according to:
    \begin{equation}
        f_e(a)=\sum_{i\in\mathbb{N}_{\le{}|e|}}\left(
            \hat{f}_e(i)\cdot{}\left(\bigotimes_{j\in{}e}[a_j, 1]\right)_i
        \right),
    \end{equation}
    where $\bigotimes$ denotes the convolution of $|e|$ sequences.
    $\hat{f}_e(i)\in\mathbb{R}^{|e|+1}$ is computed by Eq.~\eqref{eq:coefficient} using any $S$ such that $|S|=i$.
\end{corollary}

The above corollary implies that the complexity of evaluating a symmetric \ac{FE} can be reduced from $O(2^{|e|})$ to $O(|e|^2)$ using convolution.
Without loss of generality, it can be applied to different types of higher-order Ising machines as mentioned in Theorem~\ref{thm:reduction}, Corollary~\ref{cor:type2} and~\ref{cor:type3}.

\subsection{Gradient Computation}
Continuous optimizers often rely on gradient descent to find the minima of a non-convex but smooth function.
With the theoretic framework established, \texttt{IsingSim} uses both explicit gradient and estimated gradient to solve different Hamiltonian functions.

\subsubsection{Exact Gradient}
Computing the gradient requires computing the partial derivatives on all dimensions.
The following corollary implies that computing partial derivatives on one dimension is almost as expensive as the evaluation in Corollary~\ref{cor:symmetric}.
\begin{corollary}\label{cor:explicit_gradient}
    The partial derivative of $f_e$ w.r.t $x_j$ is:
    \begin{equation}
        \frac{\partial f_e(a)}{\partial x_j}=\sum_{i\in\mathbb{N}_{\le{}|e|-1}}\left(
            \hat{f}_e(i)\cdot{}\left(\bigotimes_{{e_k}\in{}e\backslash{}e_j}[a_{e_k}, 1]\right)_i
        \right).
    \end{equation}
\end{corollary}

Instead, \texttt{IsingSim} uses a more efficient way for computing the gradient, \emph{i.e.}, cumulative convolution.
When running convolution on multiple $\mathbb{R}^2$ sequences sequentially, \texttt{IsingSim} keeps track of the intermediate convolution result, which is called $\texttt{CC}_{seq}$.
\texttt{IsingSim} also runs another cumulative convolution with reverse order, which the result is called $\texttt{CC}_{rev}$.
Specifically $\texttt{CC}_{seq}[0]$ and $\texttt{CC}_{rev}[0]$ is $\emptyset$.
$\texttt{CC}_{seq}[|e|]$ and $\texttt{CC}_{rev}[|e|]$ is the final convolution result.
With the $\texttt{CC}$'s, the gradient can be obtained from the following equation.
\begin{equation*}
    \frac{\partial f_e(a)}{\partial a_{e_j}}=\sum_{i\in\mathbb{N}_{\le{}|e|-1}}\left(
        \hat{f}_e(i)\cdot{}\left(\texttt{CC}_{seq}[j-1]\ast\texttt{CC}_{rev}[|e|-j]\right)_i
    \right)
\end{equation*}
Hence, except for $\texttt{CC}_{seq}$ and $\texttt{CC}_{rev}$, the partial derivative on each dimension only requires one additional convolution, ensuring the performance of our simulator.  

\subsubsection{Estimated Gradient}
\texttt{IsingSim} integrates gradient estimators.
A common approach is the two-point estimation.
\begin{equation*}
    \frac{\partial \bar{\mathcal{H}}(a)}{\partial a_i}\approx
    \frac{\bar{\mathcal{H}}\left(a+\delta\vec{i}\right)-\bar{\mathcal{H}}(a)}{\delta},
\end{equation*}
where $\vec{i}$ is a basis vector.

\subsubsection{Approximate Gradient of the Moreau Envelope}
As $\bar{\mathcal{H}}(a)$ is often non-convex, \texttt{IsingSim} integrates an estimator for the approximate gradient of the Moreau envelope.
\begin{proposition}[Gradient Estimation~\cite{osher2023hamilton}]
    The Moreau envelop of the Hamiltonian function is 
    \begin{equation*}
        u_t(a) = \min_{b\in\mathbb{R}^n} \bar{\mathcal{H}}(b)+\frac{1}{2t}||b-a||^2.
    \end{equation*}
    Given $\alpha,\delta>0$, the gradient of $u_t$ is estimated by
    \begin{equation*}
        \nabla u_{t\mid\alpha,\delta}(a) = \frac{1}{t}\left(
            a - \frac{\mathbb{E}_{b\sim\mathcal{N}(a, \frac{\delta{}t}{\alpha})}[b\cdot\exp(-\bar{\mathcal{H}}(b)/\delta)]}{\mathbb{E}_{b\sim\mathcal{N}(a, \frac{\delta{}t}{\alpha})}[\exp(-\bar{\mathcal{H}}(b)/\delta)]}
        \right)
    \end{equation*}

\begin{remark}\label{rmk:moreau}
    The practical implementation is to sample multiple points from $\mathcal{N}(a, \frac{\delta{}t}{\alpha})$ and apply the softmax on the function value to obtain the weight of each sample point. 
    The weighted combination of sample points forms the estimated proximal.
\end{remark}

\end{proposition}

\section{Implementations and Evaluations}\label{sec:result}

In this section, we design experiments to answer the following research questions.

\noindent\textbf{RQ1.} 
By following the $\nabla\mathcal{H}$, $\tilde{\nabla}\mathcal{H}$, and $\tilde{\nabla}u$, do Ising spins converge to saddle point, local minima, or global minima?

\noindent\textbf{RQ2.}
How compact are higher-order Ising models as compared to traditional Ising models?

\noindent\textbf{RQ3.}
In higher-order Ising models encoded by practical problems, which Ising spin can achieve the best performance?

We implemented \texttt{IsingSim} using ADAM as the optimizer~\cite{kingma2014adam} with default parameter ($lr=0.05$, $\beta_1=0.9$, $\beta_2=0.999$) to implement $\nabla\bar{\mathcal{H}}$, $\tilde{\nabla}\bar{\mathcal{H}}$, and $\tilde{\nabla}u$ unless specifically stated.
Experiments are conducted on a cluster node with dual AMD EPYC 9654 CPUs and an NVIDIA H100 GPU.

\subsection{Ising Spin Trajectories}
For RQ1, we test different Ising spins and gradients on the Hamiltonian functions with a 2-XOR constraint.
The initial point is chosen so that, by following the gradient, the Ising spins will converge to a saddle point or a local minimum

\begin{figure}
    \centering
    \includegraphics[width=\linewidth]{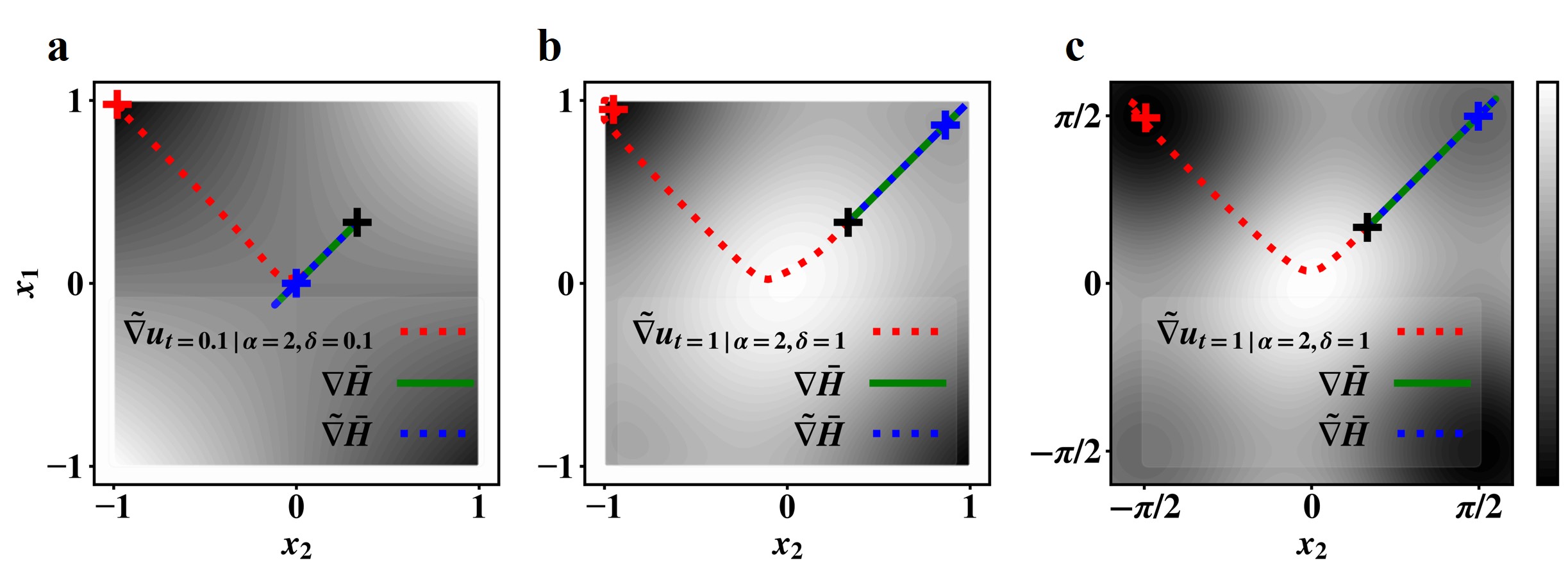}
    \caption{
    The gradient descent trajectories of two coupled \textbf{a}. type I, \textbf{b}. type II ($p=1$), and \textbf{c}. type III Ising spins under \texttt{XOR(x1,x2)}.
    Black crosses are the initial guesses.
    The crosses with other colors are the convergence points.
    }
    \label{fig:trajectory}
    \vspace{-0.5cm}
\end{figure}

\noindent\textbf{RQ1.} 
The Fourier expansion of a 2-XOR constraint is $f_\oplus = x_1{}x_2$.
The result is shown in Fig.~\ref{fig:trajectory}. 
Given a simple Hamiltonian function, $\nabla\bar{\mathcal{H}}$ and $\tilde{\nabla}\bar{\mathcal{H}}$ do not differ too much.
As mentioned in Remark~\ref{rmk:optimality}, all critical points in the type I Hamiltonian are saddle points.
It can be easily verified that $(0,0)$ is a saddle point in Fig.~\ref{fig:trajectory}~a.
If the initial point $(x_{1,0}, x_{2,0})$ such that $x_{1,0}=x_{2,0}$, then following $\nabla\bar{\mathcal{H}}$ or $\tilde{\nabla}\bar{\mathcal{H}}$ will always converge to the saddle point.
From Fig.~\ref{fig:trajectory}~b we can observe that, type II Ising spins at the same initial point will converge to a local minimum near $(1,1)$.
While type-II Ising spins might not converge to $\{\pm{}1\}^n$, the Hamiltonian function of type-III Ising spins always achieve its minima on $\left\{\pm{}\frac{\pi}{2}\right\}^n$.

While the convergence points of using $\nabla\mathcal{H}$ and $\tilde{\nabla}\mathcal{H}$ heavily depend on the initialization, $\tilde{\nabla}u$ direct the Ising spins to the global minima.
However, we can choose a smaller $t$ and $\delta$, \emph{i.e.}, a smaller sampling area, for estimating $\nabla{u}$.

\subsection{Solving Parity Learning with Error Problems}
The parity learning problem is to learn an unknown parity function given input-output samples.
When the output is noisy, \emph{i.e.,} at most half of the output is inverted, whether the problem is in P remains an open question~\cite{crawford1994minimal}, and is widely used in cryptography~\cite{pietrzak2012cryptography}.
The noisy version is called \ac{PLE} problem.
Technically, \ac{PLE} aims to find an assignment that can violate at most $e\cdot m$ out of $m$ XOR constraints.
For each $n \in \{8, 16, 32, 64\}$, \emph{i.e.}, the number of parity bits, we choose $e = 1/2$ and $m = 2n$ to generate 100 instances according to~\cite{hoos2000satlib} for RQ2 and RQ3.

\noindent\textbf{RQ2.}
The \ac{PLE} problem can be encoded into XOR constraints and 1 cardinality constraint~\cite{hoos2000satlib}, \emph{e.g.}, the hybrid SAT formula in Fig.~\ref{fig:isingsim} encodes 4 parity codes and allowing 1 fault.
Each hyperedge in higher-order Ising models directly encodes a Boolean constraint.
From Table~\ref{tab:encode} we can observe that, the size of higher-order Ising models in Definition~\ref{def:hoim} scales \emph{linearly}.

However, previous higher-order Ising machines only accept Boolean constraints that the \ac{FE} is in the form of multiplication, \emph{e.g.,} CNF~\cite{bybee2023efficient}, XOR~\cite{bhattacharya2024unified}.
We use Modulo Totalizer to encode the cardinality constraint into CNF~\cite{morgado2014mscg}.
If higher-order Ising models encodes CNF-XOR instead, the size scales \emph{polynomially}.
We encode second-order Ising models as in Example~\ref{eg:encode}~\cite{lucas2014ising}.
When $n=64$, $|V_2|$ is 12.33$\times$ larger than $|V_h|$, and $|E_2|$ is 874.97$\times$ larger than $|E_h|$.

\begin{table}[t]
\addtolength{\tabcolsep}{-5pt}
\centering
\caption{The size of Ising models which encode the parity learning with error problem.
$|V|$ is the number of spins and $|E|$ is the number of (hyper)edges.}
\begin{tabular}{>{\centering\arraybackslash}p{1.5cm} 
c >{\centering\arraybackslash}p{1cm} c >{\centering\arraybackslash}p{1cm} 
c >{\centering\arraybackslash}p{1cm} c >{\centering\arraybackslash}p{1cm} 
c >{\centering\arraybackslash}p{1cm} c >{\centering\arraybackslash}p{1.5cm}}
\toprule
\multirow{2}{*}{\parbox{1.5cm}{\centering Num. of \\ Parity Bits \\$n$}}
                  & & \multicolumn{3}{c}{HybridSAT} 
                  & & \multicolumn{3}{c}{CNF-XOR} 
                  & & \multicolumn{3}{c}{Max-2-XOR} \\
                \cmidrule{3-5}  \cmidrule{7-9} \cmidrule{11-13}
                  & & $|V_h|$      & & $|E_h|$     & & $|V_{cx}|$     & & $|E_{cx}|$     & & $|V_2|$     & & $|E_2|$  \\
\cmidrule{1-1}  \cmidrule{3-3}  \cmidrule{5-5}  \cmidrule{7-7}  \cmidrule{9-9} \cmidrule{11-11} \cmidrule{13-13}
8                 & & \textbf{24}   & & \textbf{17} & & 82          & & 181         & & 71.86       & & 668.33       \\
16                & & \textbf{48}   & & \textbf{33} & & 168         & & 371         & & 208.1       & & 3266.36  \\
32                & & \textbf{96}   & & \textbf{65} & & 337         & & 845         & & 672.13      & & 18052.38 \\
64                & & \textbf{192}  & & \textbf{129}& & 768         & & 2311        & & 2367.49     & & 112871.50\\
\bottomrule
\end{tabular}

\label{tab:encode}
\vspace{-0.5cm}
\end{table}

\noindent\textbf{RQ3.}
\ac{PLE} problems are encoded into higher-order Ising Hamiltonian $G(V,E)$.
We choose a weight $w_e=|e|$ for $e\in{}E$.
$\nabla\mathcal{H}$ and $\tilde{\nabla}\mathcal{H}$ are tested on three types of Ising spins for 500 gradient steps.
There are a total of 400 \ac{PLE} instances.
For each instance, we test with 100 trials with random initial points.
The success rate of 40000 trials is shown in Fig.~\ref{fig:ple}~a-b.
After 500 gradient steps, the \ac{IQR} and median of using different gradients are similar.
Instead, the gradient types are more impactful on the success rates.
The performance is ranked by Type I, II, then III. 
One intuition is from the weak-convexity, where for $\rho>0$, a function $f$ is $\rho$-weakly convex if $f(x)+\frac{\rho}{2}||x||^2$ is convex.
It can be verified from the Hessian $\nabla^2\bar{\mathcal{H}}+\rho{}I$.
When $\rho_1>\max_i\left(\sum_{e\ni{}i,j}\sum_{S}\left|\hat{f}_e(S)\right|\right)$, the above for Type I Hamiltonian is positive definite.
For Type II, $\rho_2\ge\rho_1+8$.
For Type III, $\rho_3>\max_i\left(\sum_{e\ni{}i}\sum_{S}\left|\hat{f}_e(S)\right| +  \sum_{e\ni{}i,j}\sum_{S}\left|\hat{f}_e(S)\right|\right)+4$.
For most instances, we have $\inf(\rho_1)<\inf(\rho_2)<\inf(\rho_3)$.

Since ADAM optimizer with $\tilde{\nabla}u$, as gradient converges slower, we choose $lr=1$.
Recall from Remark~\ref{rmk:moreau} that a larger size problem requires more samples to estimate ${\nabla}u$.
We implement $\tilde{\nabla}u$ on the smallest instances with 1000 samples per gradient estimation.
The best practice for $\tilde{\nabla}u$ is to fine-tune hyperparameters with adaptive parameters for specific problems~\cite{heaton2024global}. 
Implementing $\tilde{\nabla}u$ on high-dimension problems is less practical, where the engineering efforts might be larger than solving the problems themselves. 
Here we choose $\alpha,\delta,t=1$.
After 10000 gradient steps, the success rates are lower than using $\nabla\mathcal{H}$ and $\tilde{\nabla}\mathcal{H}$ for 500 steps.
Typically, the success rate of using Type III spins is close to 0.
We can get insights from the above analysis of weak convexity, where a larger $\rho$ will make it harder to sample outside a local minima.

\begin{figure}
    \centering
    \includegraphics[width=\linewidth]{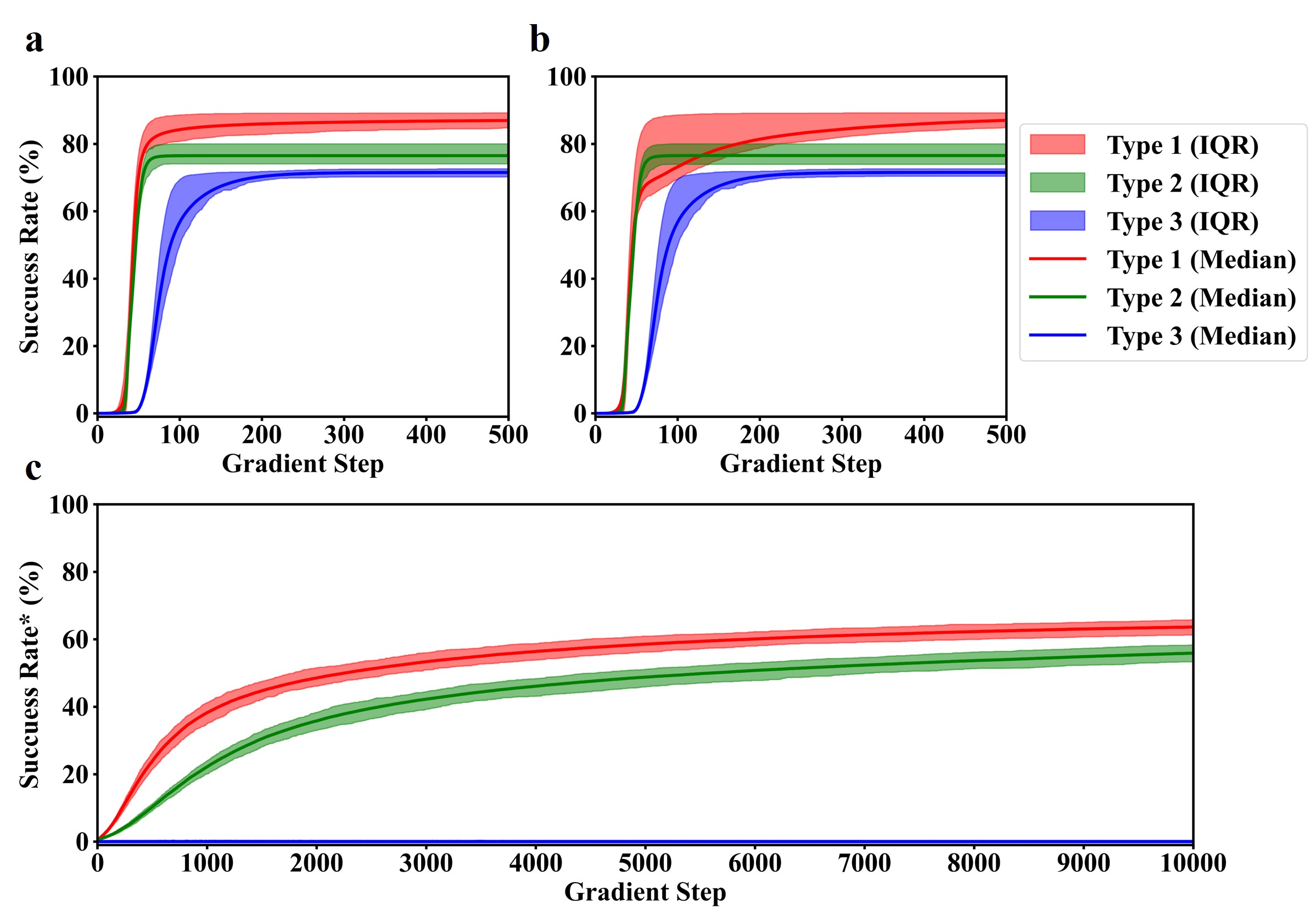}
    \caption{
    Success rates at each gradient step using \textbf{a}. $\nabla\bar{H}$, \textbf{b}. $\tilde{\nabla}\bar{H}$, and \textbf{c}. $\tilde{\nabla}u$.
    * denotes the success rates on \ac{PLE} problems with 8 parity bits.
    }
    \label{fig:ple}
    \vspace{-0.5cm}
\end{figure}

\section{Conclusion}\label{sec:conclusion}
This work proposes \texttt{IsingSim}, a customizable framework for analyzing the convergence of higher-order Ising Hamiltonian.
Corollary~\ref{cor:symmetric} and~\ref{cor:explicit_gradient} ensure the scalability for higher-order Ising machines.
Three types of Ising spins are studied in the experiments.
Results on \ac{PLE} problems show that the performance of using Type I Ising spins is higher than II or III.
The insights of the above observation are provided from the angle of weak convexity. 
Additionally, we show that the approximate gradient of the Moreau envelope $\tilde{\nabla}u$ can direct the Ising spins to global minima, and however, is less practical on higher dimension problems.
The encouraging results imply that $\nabla\mathcal{H}$ and $\tilde{\nabla}\mathcal{H}$ can be candidate methods for implementing higher-order Ising machines. 
They also highlight that \texttt{IsingSim} can be a useful tool for exploring new designs.

\bibliographystyle{IEEEtran}
\bibliography{reference}

\begin{thebibliography}{10}
\providecommand{\url}[1]{#1}
\csname url@samestyle\endcsname
\providecommand{\newblock}{\relax}
\providecommand{\bibinfo}[2]{#2}
\providecommand{\BIBentrySTDinterwordspacing}{\spaceskip=0pt\relax}
\providecommand{\BIBentryALTinterwordstretchfactor}{4}
\providecommand{\BIBentryALTinterwordspacing}{\spaceskip=\fontdimen2\font plus
\BIBentryALTinterwordstretchfactor\fontdimen3\font minus \fontdimen4\font\relax}
\providecommand{\BIBforeignlanguage}[2]{{%
\expandafter\ifx\csname l@#1\endcsname\relax
\typeout{** WARNING: IEEEtran.bst: No hyphenation pattern has been}%
\typeout{** loaded for the language `#1'. Using the pattern for}%
\typeout{** the default language instead.}%
\else
\language=\csname l@#1\endcsname
\fi
#2}}
\providecommand{\BIBdecl}{\relax}
\BIBdecl

\bibitem{kyrillidis2020fouriersat}
A.~Kyrillidis, A.~Shrivastava, M.~Vardi, and Z.~Zhang, ``Fouriersat: A fourier expansion-based algebraic framework for solving hybrid boolean constraints,'' in \emph{Proceedings of the AAAI Conference on Artificial Intelligence}, vol.~34, no.~02, 2020, pp. 1552--1560.

\bibitem{golia2022scalable}
P.~Golia, B.~Juba, and K.~S. Meel, ``A scalable shannon entropy estimator,'' in \emph{International Conference on Computer Aided Verification}.\hskip 1em plus 0.5em minus 0.4em\relax Springer, 2022, pp. 363--384.

\bibitem{wang2023fastpass}
F.~Wang, J.~Liu, and E.~F. Young, ``Fastpass: Fast pin access analysis with incremental sat solving,'' in \emph{Proceedings of the 2023 International Symposium on Physical Design}, 2023, pp. 9--16.

\bibitem{vardi2023solving}
M.~Y. Vardi and Z.~Zhang, ``Solving quantum-inspired perfect matching problems via tutte's theorem-based hybrid boolean constraints,'' \emph{arXiv preprint arXiv:2301.09833}, 2023.

\bibitem{lucas2014ising}
A.~Lucas, ``Ising formulations of many np problems,'' \emph{Frontiers in physics}, vol.~2, p.~5, 2014.

\bibitem{hauke2020perspectives}
P.~Hauke, H.~G. Katzgraber, W.~Lechner, H.~Nishimori, and W.~D. Oliver, ``Perspectives of quantum annealing: Methods and implementations,'' \emph{Reports on Progress in Physics}, vol.~83, no.~5, p. 054401, 2020.

\bibitem{borders2019integer}
W.~A. Borders, A.~Z. Pervaiz, S.~Fukami, K.~Y. Camsari, H.~Ohno, and S.~Datta, ``Integer factorization using stochastic magnetic tunnel junctions,'' \emph{Nature}, vol. 573, no. 7774, pp. 390--393, 2019.

\bibitem{aadit2022massively}
N.~A. Aadit, A.~Grimaldi, M.~Carpentieri, L.~Theogarajan, J.~M. Martinis, G.~Finocchio, and K.~Y. Camsari, ``Massively parallel probabilistic computing with sparse ising machines,'' \emph{Nature Electronics}, vol.~5, no.~7, pp. 460--468, 2022.

\bibitem{si2024energy}
J.~Si, S.~Yang, Y.~Cen, J.~Chen, Y.~Huang, Z.~Yao, D.-J. Kim, K.~Cai, J.~Yoo, X.~Fong \emph{et~al.}, ``Energy-efficient superparamagnetic ising machine and its application to traveling salesman problems,'' \emph{Nature Communications}, vol.~15, no.~1, p. 3457, 2024.

\bibitem{marandi2014network}
A.~Marandi, Z.~Wang, K.~Takata, R.~L. Byer, and Y.~Yamamoto, ``Network of time-multiplexed optical parametric oscillators as a coherent ising machine,'' \emph{Nature Photonics}, vol.~8, no.~12, pp. 937--942, 2014.

\bibitem{wang2021solving}
T.~Wang, L.~Wu, P.~Nobel, and J.~Roychowdhury, ``Solving combinatorial optimisation problems using oscillator based ising machines,'' \emph{Natural Computing}, vol.~20, no.~2, pp. 287--306, 2021.

\bibitem{wang2024design}
Y.~Wang, Y.~Cen, and X.~Fong, ``Design framework for ising machines with bistable latch-based spins and all-to-all resistive coupling,'' in \emph{2024 IEEE International Symposium on Circuits and Systems (ISCAS)}.\hskip 1em plus 0.5em minus 0.4em\relax IEEE, 2024, pp. 1--5.

\bibitem{bybee2023efficient}
C.~Bybee, D.~Kleyko, D.~E. Nikonov, A.~Khosrowshahi, B.~A. Olshausen, and F.~T. Sommer, ``Efficient optimization with higher-order ising machines,'' \emph{Nature Communications}, vol.~14, no.~1, p. 6033, 2023.

\bibitem{bhattacharya2024unified}
T.~Bhattacharya, G.~Hutchinson, and D.~B. Strukov, ``Unified framework for efficient high-order ising machine hardware implementations,'' \emph{preprint, avaiable at: https://web. ece. ucsb. edu/\~{} strukov/papers/2024/IM2024. pdf}.

\bibitem{o2021analysis}
R.~O'Donnell, ``Analysis of boolean functions,'' \emph{arXiv preprint arXiv:2105.10386}, 2021.

\bibitem{cen2023massively}
Y.~Cen, Z.~Zhang, and X.~Fong, ``Massively parallel continuous local search for hybrid sat solving on gpus,'' \emph{arXiv preprint arXiv:2308.15020}, 2023.

\bibitem{protter2012maximum}
M.~H. Protter and H.~F. Weinberger, \emph{Maximum principles in differential equations}.\hskip 1em plus 0.5em minus 0.4em\relax Springer Science \& Business Media, 2012.

\bibitem{osher2023hamilton}
S.~Osher, H.~Heaton, and S.~Wu~Fung, ``A hamilton--jacobi-based proximal operator,'' \emph{Proceedings of the National Academy of Sciences}, vol. 120, no.~14, p. e2220469120, 2023.

\bibitem{kingma2014adam}
D.~P. Kingma, ``Adam: A method for stochastic optimization,'' \emph{arXiv preprint arXiv:1412.6980}, 2014.

\bibitem{crawford1994minimal}
J.~M. Crawford, M.~J. Kearns, and R.~E. Schapire, ``The minimal disagreement parity problem as a hard satisfiability problem,'' \emph{Computational Intell. Research Lab and AT\&T Bell Labs TR}, 1994.

\bibitem{pietrzak2012cryptography}
K.~Pietrzak, ``Cryptography from learning parity with noise,'' in \emph{International Conference on Current Trends in Theory and Practice of Computer Science}.\hskip 1em plus 0.5em minus 0.4em\relax Springer, 2012, pp. 99--114.

\bibitem{hoos2000satlib}
H.~H. Hoos and T.~St{\"u}tzle, ``Satlib: An online resource for research on sat,'' \emph{Sat}, vol. 2000, pp. 283--292, 2000.

\bibitem{morgado2014mscg}
A.~Morgado, A.~Ignatiev, and J.~Marques-Silva, ``Mscg: Robust core-guided maxsat solving,'' \emph{Journal on Satisfiability, Boolean Modeling and Computation}, vol.~9, no.~1, pp. 129--134, 2014.

\bibitem{heaton2024global}
H.~Heaton, S.~Wu~Fung, and S.~Osher, ``Global solutions to nonconvex problems by evolution of hamilton-jacobi pdes,'' \emph{Communications on Applied Mathematics and Computation}, vol.~6, no.~2, pp. 790--810, 2024.

\end{thebibliography}

\end{document}